\documentclass[letterpaper, 10 pt, conference]{ieeeconf}  

\IEEEoverridecommandlockouts                              
\usepackage{booktabs}
\usepackage{algorithm}
\usepackage[noend]{algpseudocode}
\usepackage{array}
\usepackage{graphicx}
\usepackage{epstopdf}
\usepackage{cite}
\usepackage{amsfonts,amssymb}  
\usepackage{amsmath}
\usepackage{color}
\usepackage[export]{adjustbox}

\newtheorem{thm}{Theorem}[section]
\newtheorem{lm}[thm]{\bf Lemma}
\newtheorem{prob}[thm]{\bf Problem}
\newtheorem{df}[thm]{\bf Definition}
\newtheorem{pr}[thm]{\bf Proposition}
\newtheorem{rem}[thm]{\bf Remark}

\newtheorem{asms}[thm]{\bf Assumptions}
\newtheorem{asm}[thm]{\bf Assumption}

\title{\LARGE \bf
Event-Triggered Controller Synthesis for Dynamical Systems \\ with Temporal Logic Constraints
}

\author{Dipankar Maity and John S. Baras
\thanks{The authors are with the Department of Electrical and Computer
        Engineering and The Institute for Systems Research,
        University of Maryland, College Park, USA. Email:
        {\tt\small dmaity@umd.edu, baras@umd.edu}}%
}

\begin{document}

\maketitle
\thispagestyle{empty}
\pagestyle{empty}

\noindent
\begin{abstract}
In this work, we propose an event-triggered control framework for dynamical systems with temporal logical constraints.
Event-triggered control methodologies have proven to be very efficient in reducing sensing, communication and  computation
costs. When a continuous feedback control is replaced with an event-triggered strategy, the corresponding state trajectories 
also differ. In a system with logical constraints, such small deviation in the trajectory might lead to  unsatisfiability of the
logical constraints. In this work, we develop an approach where we ensure that the event-triggered state trajectory is confined within
an $\epsilon$ tube of the ideal trajectory associated with the continuous state feedback. At the same time, we will ensure satisfiability of the
logical constraints as well. 
Furthermore, we show that the proposed method works for delayed systems as long as the delay is bounded by a certain quantity.
\end{abstract}

\section{Introduction}
Present control systems are typically a large network of heterogeneous components sharing some common resources and information, and with coordinated cooperation, they aim to achieve desired performance.  These kinds of  highly complex systems are ubiquitous in cyber-physical-systems (CPS), and also referred as networked-CPS. In many CPS, the controller synthesis is subjected to many logical constraints that arise due to presence of logical variables and reasoning among the subsystems. 
Recent studies on controller synthesis with linear temporal logic (LTL) have paved a way to design controllers for large complex systems with safety, synchronisation, and other logical constraints \cite{bemporad1999control,antoniotti1995discrete,kloetzer2008fully}. 

Novel formulations and efficient computational approaches have been proposed to mathematically formulate specifications such as trajectory sequencing, synchronization etc. 
Temporal logics such as linear temporal logic (LTL), computational tree logic (CTL), developed for model checking, have been widely accepted by the robotics community for the purpose of motion planning 
\cite{fainekos2009temporal}, \cite{ulusoy2012robust}. Development of sophisticated model checking tools such as SPIN 
and NuSMV 
made it easier to synthesize controllers for such systems. As an alternative approach controller synthesis has been done using mixed integer linear programming \cite{bemporad1999control}. 

Another challenge for the large connected CPS is that computation of the control law requires continuous sensing (often times distributed) and transmitting the sensed signals to the controllers. Consequently, the performance of such systems is generally determined by the availability of sensing power, bandwidth for continuous transmission and resources for fast computation. Therefore it will be beneficial if  the same (with little tolerance) performance can be achieved with lesser intensive sensing and computing tasks.

To circumvent the problem of limited communication bandwidth or computing resources or sensing capability, researchers have developed techniques that require intermittent communications only at certain discrete time instances to perform the same task with minor performance degradation. These control methodologies are known in many forms e.g.  event-triggered, self-triggered  or periodic control \cite{heemels2012introduction}, \cite{bian2005general}. These control strategies do not require the state information $x(t)$ for all time $t$, rather they sample  $x(t)$ intermittently depending on the systems' performance criterion \cite{maity2015event}, \cite{maity2015cdc}. These  techniques have proven to be efficient for large scale inter-connected systems to reduce communication and sensing operations.

In this work, we study the temporal logic based controller synthesis problem in an event-triggered framework. We consider a controller synthesis problem for a given control affine nonlinear system and the objective is to design an event triggered controller for that system with logical constraints. We assume the logical constraints can be represented using temporal logic and its propositional calculus. 

Existing literature results show that the trajectory of an event triggered system deviates from the nominal system as a consequence of limited communication \cite{maity2015event}, \cite{maity2015cdc}. Although, the continuous feedback system satisfies the logical constraints, now with an event triggered controller we have no guarantee that the logical constraint over the event-triggered trajectory will be satisfied as well. 

We show that suitably modifying the given logical constraints, and creating  stricter constraints will make the event-triggered trajectory satisfy the original logical constraint provided we can synthesize a continuous controller to satisfy the stricter constraint. The stricter logical constraint is often times known as  robust logical constraint \cite{fainekos2009robustness} since any  perturbed trajectory (within some $\epsilon$ bound) will still satisfy the constraint. We adopt this notion of robustness in this work. In the next stage, we design an event-triggered controller ensuring that the event-triggered controller confines the trajectory within an $\epsilon$-tube around the actual trajectory. Further, we show the effects of delay (in transmitting the measurement to the controller) on the performance. The analysis shows that if the delay is bounded by a certain quantity, which depends on the physical parameters of the plant and the controller, then the delayed system will be able to perform similar to the delay-free system without further modification in design.

In Section \ref{S:2}, we formally describe the problem and our two-step approach towards the problem. Section \ref{S:3} provides preliminary background on the temporal logic and construction of $\epsilon$-robust logic formulae. We design an event triggered controller for this problem in Section \ref{S:4} and study the effects of delay on such an event triggered controller. Finally, we illustrate the application of our framework using two examples in Section \ref{S:5}. 

\section{Problem formulation} \label{S:2}
Let us consider the input-affine nonlinear state space model as given in (\ref{Eqn::Dynamics}). \!
\begin{align} \label{Eqn::Dynamics}
&\dot x =f_0(t,x)+\sum\limits_{i=1}^m f_i(t,x)\cdot u_i \\
\vspace{-2pt}
&x(t_0)=x_0. \nonumber
\end{align}
where $x(t) \in \mathcal{X} \subseteq \mathbb{R}^n$, $u_i(t) \in \mathbb{R}$ is the $i$-th control input. The trajectory of the dynamical system starting at $t_0$ under application of some control $u=[u^1,\cdots,u^m]$ is denoted as $x^{u,x_0}[t_0]$. Similarly the trajectory starting from any arbitrary point $(t,x)$ is represented as $x^{u,x}[t]$. The objective of this work is to design an event-triggered controller $u(\cdot)$ that ensures satisfiability of a temporal logical constraints ($\varphi$). 

By $x^{u,x_0}[t_0]\models \varphi$ we denote that the trajectory of the dynamics (\ref{Eqn::Dynamics}) under input $u$ satisfies the logical constraint $\varphi$. Similarly, $x^{u,x_0}[t_0]\not\models \varphi$ denotes that the trajectory does not satisfy the logical constraint. In this work we focus on the real time linear temporal logics \cite{fainekos2009temporal}, \cite{reynolds2001continuous} which have been proven to be very effective for expressing logical constraints in dynamical systems \cite{bemporad1999control}. 
 In the following, we formally pose the problem that we aim to solve in this work.
 
\begin{prob} \label{P}
Given an input-affine dynamics (\ref{Eqn::Dynamics}) and a logical constraint ($\varphi$) on the trajectory of the system, design an event-triggered framework to generate the control $u(t)$ such that the event-triggered trajectory satisfies $\varphi$, i.e. 
\begin{align} \label{P:opt}
\text{find~~~~}  &u\\\vspace{-2pt}
\text{subject to ~~}    &u\in \mathcal{U}^e \nonumber \\\vspace{-2pt}
& \dot x =f_0(t,x)+\sum\limits_{i=1}^m f_i(t,x)\cdot u_i \nonumber\\\vspace{-2pt}
& x^{u,x_0}[t_0]\models \varphi \nonumber\\\nonumber\vspace{-2pt}
\end{align}
\end{prob} 
where $\mathcal{U}^e$ denotes the set of event based control strategies. 
In this work we do not impose any restriction on the event-triggered framework other than the exclusion of Zeno behavior \cite{ames2005sufficient}. We provide sufficient conditions which, if satisfied, ensure that the trajectory of the designed event-triggered system will satisfy the logical constraint $\varphi$. 

While temporal logic can express various types of logical constraints (see Section \ref{S:3}), it comes with a cost that verifying whether a trajectory satisfies the logical constraints is {\sc{Pspace}}-complete \cite{reynolds2001continuous}. Therefore synthesis of a controller is a hard problem in its own right, and synthesis of an event-based controller is harder for obvious reasons. However, there are some proposed techniques which can generate a (feedback) controller that satisfies the temporal logic constraints, see for example \cite{belta2007symbolic}, \cite{conner2006integrated}.

Therefore, we divide the original problem into two subproblems: Problem \ref{P:1} and Problem \ref{P:2}.
\begin{prob} \label{P:1}
Given the dynamics (\ref{Eqn::Dynamics}), design a feedback controller $u_i(t)=\gamma_i(t,x(t))$ such that $x^{\gamma,x_0}[t_0]\models \varphi^\epsilon$.
\end{prob} 
where $\varphi^\epsilon$ is another logical constraint  derived from $\varphi$.  $\varphi^\epsilon$ is a stricter constraint than $\varphi$ in the sense that $x^{\gamma,x_0}[t_0]\models \varphi^\epsilon$ implies $\xi[t_0]\models \varphi^\epsilon$ for all piecewise continuous curves $\xi(\cdot) : [t_0, +\infty)\rightarrow \mathbb{R}^n$ such that $\sup_{t\in [t_0, +\infty)}\|x(t)-\xi(t)\| \le \epsilon$. In the following sections we will explicitly explain how $\varphi^\epsilon$ is related to  $\varphi$ for a given $\epsilon \ge 0$.
 
\begin{prob} \label{P:2}
For all $\epsilon>0$, given the dynamics (\ref{Eqn::Dynamics}) and a feedback control $\gamma(t,x(t))$, design an event-triggered controller $ \gamma^e(t,x(\tau_k))$ such that the trajectory of the event-triggered system $x^{\gamma^e, x_0}[t_0]$ remains within an $\epsilon$ neighborhood of the ideal trajectory associated with the feedback closed-loop system. As $\epsilon \rightarrow 0$, $\gamma^e(t,\cdot)\rightarrow \gamma(t,\cdot)$ pointwise $\forall t$.
\end{prob} 

Therefore, in this two step approach, we first design a feedback controller for satisfying the $\epsilon$-strict constraint $\varphi^\epsilon$ (for some $\epsilon>0$). In the next stage we use an event-triggering mechanism which provides sufficient condition(s) for ensuring that the trajectory of the event-triggered system will be in an $\epsilon$ neighborhood of the actual feedback trajectory pointwise, i.e. $\|x_e(t)-x(t)\|\le \epsilon$ ($x_e$ is the event-triggered trajectory and $x$ is the ideal feedback trajectory) for all $t$. In Figure \ref{F:schematic}, we present our schematic for event-triggered controller synthesis using the proposed two-step approach. From this point onward, we will suppress the control and initial state in denoting a trajectory when these are apparent from the context i.e. we will represent $x^{u,x_0}[t_0]$ as $x[t_0]$ etc. 


\begin{figure}
\begin{center}
\includegraphics[width=0.5\textwidth]{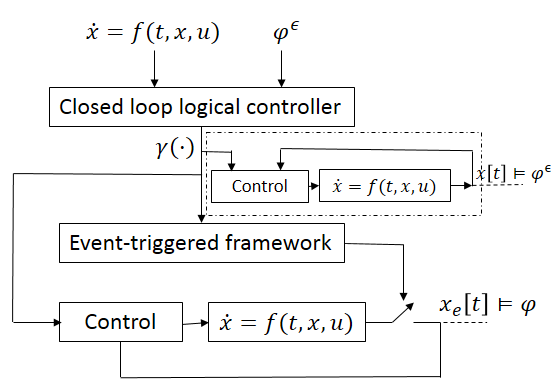}
\caption{Schematic of two-step event-triggered controller synthesis with logical constraints} \label{F:schematic}
\end{center}
\end{figure}

 \section{Propositional Temporal Logic} \label{S:3}
Like other families of propositional logic, temporal logic over the reals also requires a set of propositional variables $\Pi=\{\pi_1,\pi_2,\cdots,\pi_n\}$. Associated with each propositional variable $\pi_i$, there is a labelling function $\mathcal{L}_{i}: \mathcal{X}\rightarrow \{0,1\}$ which denotes whether the proposition $\pi_i$  is true at some point in $\mathcal{X}$. Therefore, $\Pi$  divides $\mathcal{X}$ into subsets and assigns $\pi_i$ with each of the subsets. A formula of a propositional logic is defined over a Boolean signal and in our case, $\mathcal{L}_i(\cdot)$ maps the $\mathbb{R}^n$ valued signal ($x(t)$) to a Boolean signal. For example, $\pi_1$ could be associated with the ball of radius $1$ at the origin of $\mathcal{X}=\mathbb{R}^2$. Then $\mathcal{L}_1(s)=1$ for all $x \in B_0(1)$ and 0 otherwise, where  $ B_x(\delta)=\{y\in \mathcal{X}~|~ \|y-x\|_2\le \delta\}$ is a ball of radius $\delta$ centered at $x$.  Note that it is not necessary that the regions associated with $\pi_i$ are non-overlapping. We will use the notation $\pi_i \cong \mathcal X_i (\subseteq \mathcal{X})$ to denote $\mathcal{L}_i(x)=1$ for all $x\in \mathcal{X}_i$ and $\mathcal{L}_i(x)=0$ for all $x\in \mathcal{X}\setminus \mathcal{X}_i$ (basically the indicator function of the set $\mathcal{X}_i$). At this point, it should be noted that any algebraic constraint on $x$ of the form $G(x) \le 0 $ could be associated with a proposition $\pi$ such that $\pi \cong \{x \in \mathcal{X}~|~ G(x) \le 0\}$. 

However, the power of temporal logic is beyond capturing these algebraic constraints. Let us first provide an informal overview of the capability of the logic and then formally state the syntax and semantics  of the logic. The RTL (Temporal Logic over Reals) \cite{reynolds2010complexity} formulae are built on the propositional variables $\Pi$ with the use of usual logical operators $\neg$ (negation), $\vee$ (conjunction) and $\wedge$ (disjunction), and some special temporal operators e.g. $\textbf{U}$ (until), $\Diamond$ (eventually), $\Box$ (always) and other operators that could be derived from the mentioned operators. For example, the formula $\Diamond \Box \pi$ (read as ``Eventually Always in $\pi$") where $\pi \cong \mathcal{X}_i$, when satisfied by a trajectory of the dynamics (\ref{Eqn::Dynamics}), means that eventually the trajectory enter the region $\mathcal{X}_i$ and stay there for all future times. Similarly, $ (\neg \pi_1 \wedge \neg \pi_2 \wedge \neg \pi_3)\textbf{U} \pi_4 $ states the rule that region $\mathcal{X}_4 $ must be reached   while avoiding regions $\mathcal{X}_i$ for $i=1,2,3$ ($\pi_j \cong \mathcal{X}_j$). Although the satisfaction of the formula tells us that the state trajectory will reach $\mathcal{X}_4$,  it does not provide any interval of time within which it will reach the destination. This limitation can easily be circumvented by the traditional augmentation of a new state $x_{n+1}=t$.

Adding an extra equation $\dot x_{n+1}=1$ with $x_{n+1}(t_0)=t_0$ in the dynamics (\ref{Eqn::Dynamics}), we can pose time dependent constraints as well. In this case, the augmented space is $\mathcal{X}\times [t_0,T)$ ($T$ could be $+\infty$ for an infinite horizon problem). The formula $\Box (\pi_1 \vee  \pi_2)$ where $\pi_1 \cong \{(x,t)~|~ x\in \mathcal{X}, t_0\le t < 6 \}$ and $\pi_2 \cong \{(x,t)~|~ x\in \mathcal{X}_i, 5\le t \}$ requires the trajectory to enter the region $\mathcal{X}_2$ no earlier than $t=5$ and the trajectory should remain  within $\mathcal{X}_2$ for all $t \in [6,\infty)$. The formula also mentions that the trajectory will be in $\mathcal{X}_1$ when it is not in $\mathcal{X}_2$. Therefore, with this state-space augmentation, all the properties related to the timing aspect of a trajectory of the system (\ref{Eqn::Dynamics}) can be expressed.

\begin{df} \label{def1}
 \textit{The syntax of RTL formulas are defined according to the following grammar rules:}
 \begin{center}
 $\phi ::= \top ~| ~\pi~ |~\neg \phi~ | ~\phi \vee \phi ~|~\phi \mathbf{U} \phi ~|~\phi \mathbf{R} \phi ~ $
 \end{center} 
 \end{df}
 where $\pi \in \Pi$,
$\top$ and $\bot(=\neg\top)$ are the Boolean constants {\tt{true}} and {\tt{false}} respectively.  $\mathbf{R}$ symbolizes the  $Release$ operator.
Other temporal logic operators  can be represented using the grammar in definition \ref{def1} e.g. eventually ($\Diamond \varphi=\top \textbf{U}\varphi $), always ($\Box\varphi = \neg\Diamond (\neg \varphi)$) etc. 

If $x[t_0]$ denotes a trajectory starting at time $t_0$, the semantics of the grammar in Definition \ref{def1} is given as follows:
\begin{df}\label{ltlsym}
 \textit{The semantics of any formula $\phi$ over the trajectory $x[t_0]$ is recursively defined as:\\
 $x[t_0] \models \pi$ iff $\mathcal{L}_\pi(x(t_0))=1$\\
 $x[t_0] \models \neg \pi$ iff $\mathcal{L}_\pi(x(t_0))=0$\\
 $x[t_0] \models \phi_1\vee \phi_2$ iff $x[t_0] \models \phi_1$ or $x[t_0] \models \phi_2$\\
 $x[t_0] \models \phi_1\wedge \phi_2$ iff $x[t_0] \models \phi_1$ and $x[t_0] \models \phi_2$\\
 $x[t_0] \models \phi_1\mathbf{U} \phi_2$ iff $\exists s\geq t_0$ s.t. $x[s] \models  \phi_2$ and 
 $\forall$ $t_0 \leq s'<s, ~ x[s'] \models \phi_1$.\\
 $x[t_0] \models \phi_1\mathbf{R} \phi_2$ iff $\forall s\geq t_0$  $x[s] \models  \phi_2$ or $\exists s'$ s.t. $t_0 \leq s'<s, ~ x[s'] \models \phi_1$.} 
\end{df}

More details on RTL grammar and semantics can be found in \cite{reynolds2001continuous}, \cite{fainekos2009temporal}.

\subsection{Construction of $\epsilon$-Robust Formula}

As described in Section \ref{S:2}, the motivation behind constructing an $\epsilon$-robust formula $\varphi^\epsilon$ is that any trajectory satisfying the stricter formula $\varphi^\epsilon$ is robust in the sense that any perturbed trajectory with less than $\epsilon$ perturbation will also satisfy the original constraint $\varphi$. 

The idea is as follows: if the trajectory needs to visit a region $\pi_i(\cong \mathcal{X}_i)$, then we push the boundary of $\mathcal{X}_i$ inwards by amount $\epsilon$ and denote this new set (and proposition) by $\mathcal{X}_i^\epsilon$ ($\pi_i^\epsilon$). Similarly if the trajectory needs to avoid some region $\pi_j (\cong \mathcal{X}_j)$ then the boundary of $\mathcal{X}_j$ is expanded outwards by an amount $\epsilon$. 

The RTL syntax presented in Definition (\ref{def1}) is in negative normal form (NNF) \cite{clarke1999model}, and this enables us to detect which regions must be avoided by noting the presence of the negation ($\neg$) operator immediately before the corresponding propositions $\pi_i$. 

Note that, by our definition $\pi_i \cong \mathcal{X}_i$ and $\neg \pi_i \cong \mathcal{X}\setminus \mathcal{X}_i$.

\begin{df}
\textit{
In a given metric space ($\mathcal{X},\rho$) the open ball centered at $x \in \mathcal{X}$ of radius $r$ is defined as $B_x(r)=\{y \in \mathcal{X}~|~ \rho(x,y)<r\}$. Let $\epsilon>0$ be a given parameter, then the $\epsilon$-contraction of the set $\mathcal{Y}\subseteq \mathcal{X}$ is denoted by $\mathcal{Y}^\epsilon=\{y\in \mathcal{Y}~|~ B_y(\epsilon)\subseteq \mathcal{Y}\}$.\\
Similarly, the $\epsilon$-expansion of the set is denoted by $\mathcal{Y}^{-\epsilon}=\{x\in \mathcal{X}~|~ \exists y\in \mathcal{Y}, x \in B_y(\epsilon)\}$.
}
\end{df}

For a RTL formula $\varphi$, the $\epsilon$-robust formula is constructed as follows:\\
1) replace each $\pi_i (\cong \mathcal{X}_i)$, which is not preceded by any negation, by $\pi_i^\epsilon$ where $\pi_i^\epsilon \cong \mathcal{X}_i^\epsilon$.\\
2) any $\pi_j$ that is preceded by a negation ($\neg$) should be replaced by $\pi_j^{-\epsilon} \cong \mathcal{X}_j^{-\epsilon}\cup (\mathcal X\setminus \mathcal{X}^\epsilon)$. When $\mathcal{X}=\mathcal{X}^\epsilon$ for all finite $\epsilon >0$ (e.g. $\mathcal{X}=\mathbb{R}^n$),  $\pi_j^{-\epsilon} \cong \mathcal{X}_j^{-\epsilon}$.

In a similar way one can define $\epsilon$-robust formulas over the space $\mathcal{X}\times [t_0,T)$. However, for this work we will restrict ourselves to the robustness only in $\mathcal{X}$ space.

\begin{pr} \label{Pr:robust}
Let the 
trajectory $x[t_0]$ satisfies the RTL formula $\varphi^\epsilon$ for some $\epsilon>0$. Then for all $\delta \le \epsilon$ and for any curve $y(\cdot):[t_0,T)\rightarrow \mathcal{X}$ such that $\sup_t \rho(x(t),y(t)) \le \delta$, $y[t_0]\models \varphi$.
\end{pr}

The above proposition can be proved inductively; interested readers may see \cite{fainekos2009temporal} for a proof.

\begin{rem}
In order to construct a  $\pi_i^\epsilon$ from $\pi_i$, $\mathcal{X}_i$ must have a non-empty interior. 
\end{rem}

Let us define the radius of a set $\mathcal{X}_i$ in the following way $r(\mathcal{X}_i)=\sup\{r~|~\exists x\in \mathcal{X}_i, B_x(r)\subseteq \mathcal{X}_i \}$.
Therefore for each $\pi_i$, $\pi_i^\epsilon$ is well defined for  $\epsilon\le r(\mathcal{X}_i)$, for $\epsilon > r(\mathcal{X}_i)$, $\pi_i^\epsilon \cong \emptyset$. 
For the subsequent section we will implicitly assume that  $\min_i\{r(\mathcal{X}_i)\}>0$  and moreover, each of these sets, $\mathcal{X}_i$ is a polyhedron.

Therefore, to solve problem \ref{P:1}, our goal will be to design a controller that satisfies the robust RTL formula $\varphi^\epsilon$. The power of using these temporal logic formulae is that each formula can be represented by an equivalent (B\"{u}chi) automaton. Satisfaction of an RTL formula is equivalent to finding a path from the initial state to one of the accepting states of the automaton. The construction of such automata can be done automatically using the available tools SPIN and NuSMV \cite{holzmann1997model}, \cite{cimatti2002nusmv}. The dynamics (\ref{Eqn::Dynamics}) can be represented as a finite transition system (FTS), where the transitions are performed by selecting the control $u$. Thus, the dynamics constraint and logical constraint in (\ref{P:opt}) can jointly be represented by forming a product of the automata and the FTS. More detail on such construction and controller synthesis can be found in our earlier work \cite{maity2015motion}. To keep the paper short, we spare the details here, however, interested reader may see, for example, \cite{lindemann2005smoothly}, \cite{conner2006integrated}. 

\section{Event Triggered controller synthesis} \label{S:4}
In the previous section it is presented how the control inputs can be generated for the dynamics (\ref{Eqn::Dynamics}) so that the trajectory of the system satisfies the logical constraint $\varphi^\epsilon$ for some $\epsilon>0$. This section will focus on designing an event-triggered controller  that will replace the feedback controller designed to satisfy $\varphi^\epsilon$  in such a way that the trajectory of the event-triggered system will remain within $\epsilon$ distance of the ideal feedback trajectory.

Let us denote the controller $u(t)=\gamma(x(t))$ that achieves the satisfaction of  $\varphi^\epsilon$. The closed-loop dynamics are:
\begin{align} \label{Eqn::Dyn2}
&\dot x =f_0(t,x) +\sum_{i=1}^m f_i(t,x)\gamma_i(x)\\
&x(t_0)=x_0 \nonumber
\end{align} 
Let $x[t_0]$ denote the trajectory of the above closed loop system. 
We make the following assumptions on the system (\ref{Eqn::Dyn2}).
\begin{asms} \label{AS:1}
\textbf{(A1)} $\gamma_i(\cdot)$ and for all $t$, $f_i(t,\cdot)$  are Lipschitz functions with Lipschitz constants $L_{\gamma}^i$,
$L_f^i$ respectively, for all $i=0,1,2,\cdots, m$. 

\textbf{(A2)} 
 For all $i=1,2\cdots,m$ and $\forall t$, $f_i(t,x)\gamma_i(x)$ and $f_0(t,x)$ are continuously differentiable functions w.r.t $x$ with continuous first derivative.
\end{asms}


In event-triggered framework, the controller is designed to be:
\begin{equation}
\gamma_i^e(t)= \gamma_i(x(t_k)) ~~~~\forall t \in [t_k,t_{k+1})
\end{equation}
where $t_k$s are the event-triggering times. An event-generator needs to be designed that will produce the $t_k$ in certain way that is explained in the following. 

Let us denote the event-triggered closed loop system as $x_e(t)$ and the corresponding trajectory as $x_e[t_0]$. Thus,
\begin{align} \label{E:edyn}
&\dot x_e= f_0(t,x_e)+\sum_{i=1}^mf_i(t,x_e)\gamma^e_i(x_e(t_k))\\
&x_e(t_0)=x_0.
\end{align}
Let us define the error $e(t)=x(t)-x_e(t)$.
Note that the event-triggered controller $\gamma_i^e(t,\cdot)$ for all $t$ is an approximation of the ideal feedback controller $\gamma_i(t,\cdot)$ by piecewise constant functions. Therefore, designing  the $\gamma_i$ at first makes the problem tractable for generating the event-triggered controller. 

The dynamics of $e(t)$ is given as:
\begin{align}
&\dot e =F(t,x)-F(t,x_e)+\sum_{i=0}^mf_i(t,x_e)(\gamma^e_i(x_e(t_k))-\gamma_i(x_e(t))) \nonumber \\
& e(t_0)=0,
\end{align}
where $F(t,x)=f_0(t,x) +\sum_{i=1}^m f_i(t,x)\gamma_i(x)$
\begin{asm} \label{AS:exp}
$y(t)$ is exponentially stable with
\begin{align}
&\dot y= A(t)y
\end{align}
where $A(t)=\frac{\partial F(t,x)}{\partial x}\Big|_{x=x(t)}$ and  $x(t)$ is the trajectory of (\ref{Eqn::Dyn2}).  
\end{asm}
We can write, 
\begin{align}
\dot e= A(t)e+\tilde f(t,x_e,e)e+\delta(t)
\end{align}
where $\tilde{f}(t,x_e,e)e=F(t,x)-F(t,x_e)-A(t)e$ and $\delta(t)=\sum_{i=0}^mf_i(t,x_e)(\gamma^e_i(x_e(t_k))-\gamma_i(x_e(t)))$

Using Assumption \ref{AS:1} (A2),
$$F(t,x+h)=F(t,x)+\int_{s=0}^1d{F(t,x+sh)}dsh$$ where $dF(t,\cdot): \mathbb{R}^n\rightarrow \mathbb{R}^n$ is a linear map which is the derivative of  the map $F(t,\cdot): \mathbb{R}^n\rightarrow\mathbb{R}^n$. Selecting $h=x_e-x=-e$, we obtain:
$$F(t,x_e)=F(t,x)-\int_0^1d{F(t,x-se)}dse.$$
Therefore,
$$\tilde{f}(t,x_e,e)=\int_{s=0}^1d{F(t,x_e(t)+(1-s)e(t))}ds-A(t).$$ 
Also $\tilde f(t,x_e,0)=0$.

Since $\dot y=A(t)y$ is the linearization of the system  $\dot e=A(t)e+\tilde{f}(t,x_e,e)e$ around $e=0$, we can say that $\dot e=A(t)e+\tilde{f}(t,x_e,e)e$ is locally exponentially stable due to Assumption \ref{AS:exp}. As a consequence of the Lyapunov converse theorem \cite[ Theorem 4.14]{khalil1996nonlinear}, we have a (local) quadratic Lyapunov function that satisfies:
\begin{align}
c_1\|e\|^2 \le V(t,e) &\le c_2\|e\|^2\\
\frac{\partial V}{\partial t} +\frac{\partial V}{\partial e}(A(t)+\tilde{f}(t,x_e,e))e&\le -c_3\|e\|^2\\
\Big\|\frac{\partial V}{\partial e}\Big\| &\le c_4\|e\|
\end{align}

\begin{pr} \label{pr:e}
For all $t$,
\begin{align} \label{E:boundedE}
\|e(t)\|_2\le \frac{c_4}{2c_1}\int\limits_{t_0}^te^{-(t-s)c_3/2c_2}\|\delta(s)\|_2ds
\end{align}
\end{pr}

\begin{proof}
A detailed proof of this can be found in our earlier work \cite[Theorem 2.6]{maity2015cdc}. The proposition is due to the BIBO (bounded input bounded output) stability of an exponential stable system.
\end{proof}

Note that at each $t_k$, $\delta(t_k)=0$. We can bound $\|\delta(t)\|$ to ensure a bound on $\|e(t)\|$.

\begin{rem}
 Without computing $x(t)$ real-time, $e(t)$ can be bounded by observing the signal $\delta(t)$ which depends only on $x_e$. Moreover, due to the Lipschitz assumptions on $f_i(t,\cdot)$ and $\gamma_i(\cdot)$, it is sufficient to only monitor the difference signal $x(t)-x(t_k)$.
\end{rem}
From (\ref{E:boundedE}), we have the sufficiency condition that $\sup_t\|\delta(t)\|\le\epsilon_1= \frac{c_1c_3}{c_2c_4}\epsilon$ ensures $\sup_t\|e(t)\|\le \epsilon$.

We propose the following event-trigger function:
\begin{align} \label{E:event}
g(t)= \frac{c_1c_3}{c_2c_4}\epsilon-\|\delta(t)\|.
\end{align}
An event is generated whenever $g(t)\le 0$ and the state value at that time ($x(t_k)$) is sent to the controller. The set of triggering times is denoted by $\mathcal T =\{t_1,t_2,\cdots ,t_k,\cdots\}$ such that $g(t_k)=0$ and $g(t)<0$ otherwise.

There could be other event-triggering functions that can also ensure bounded error $e(t)$. In this paper, we consider (\ref{E:event}) to carry out the analysis further and to perform the simulations.

\begin{lm} \label{l:del}
For all $t$, 
\begin{align} \label{E;dq}
\|\delta(t)\| \le \alpha\|x_e(t)-x_e(t_k)\|^2+\beta(t)\|x_e(t)-x_e(t_k)\|. 
\end{align}
for some $\alpha,\beta(t)>0$.  $t_k$ is the latest triggering time at time $t$. 
\end{lm}
\begin{proof}
We have $\delta(t)=\sum_{i=1}^mf_i(t,x_e)(\gamma^e_i(x_e(t_k))-\gamma_i(x_e(t)))$.
By rearranging,
\begin{align*}
&\delta(t)=\sum_{i=1}^mf_i(t,x_e(t_k))(\gamma^e_i(x_e(t_k))-\gamma_i(x_e(t)))+\\
&\sum_{i=1}^m(f_i(t,x_e(t))-f_i(t,x_e(t_k)))(\gamma^e_i(x_e(t_k))-\gamma_i(x_e(t)))
\end{align*}
Using the Lipschitz continuity assumption in Assumption \ref{AS:1}, we can write:
\begin{align}
\|\delta\|\le &\|(x_e(t_k))-(x_e(t))\|\sum_{i=1}^mL^i_\gamma\|f_i(t,x_e(t_k))\|+\nonumber \\
&\|(x_e(t_k))-(x_e(t))\|^2\sum_{i=1}^mL^i_\gamma L^i_f.
\end{align}
Now we define:

$\alpha=\sum_{i=1}^mL^i_\gamma L^i_f$ and $\beta(t)=\sum_{i=1}^mL^i_\gamma\|f_i(t,x_e(t_k))\|$.
\end{proof}

In \cite{maity2015cdc} a different bound on $\delta$ was derived. There it was shown that: \vspace{-2 mm}
\begin{align} \label{E:dl}
\|\delta(t)\| \le \kappa(t) \|(x_e(t_k))-(x_e(t))\|.
\end{align}
$\kappa(t)=\max_i\{L^i_\gamma\}\sup_{x\in \Omega} \sum_{i=1}^m \|f_i(t,x)\|$. Where the trajectory $x(t)$ of (\ref{Eqn::Dyn2}) is bounded in the domain $\Omega^\epsilon$.

Comparing (\ref{E:dl}) with (\ref{E;dq}), we notice that the former is bounded linearly w.r.t. $\|(x_e(t_k))-(x_e(t))\|$ whereas the later is bounded by a quadratic form of $\|(x_e(t_k))-(x_e(t))\|$. In most of the practical applications $\epsilon\ll 1$ and hence $\|(x_e(t_k))-(x_e(t))\|$ is required to keep smaller than $\epsilon$ (see Proposition \ref{suff}). Therefore, $\|(x_e(t_k))-(x_e(t))\|^2$ can be bounded by $\|(x_e(t_k))-(x_e(t))\|$ with proper coefficient. Furthermore, the presence of $\sup$ operator over the whole domain $\Omega^\epsilon$ in (\ref{E:dl}) implies that  $\kappa(t) \ge \beta(t)$ (in general $\kappa(t) \gg \beta(t)$). Therefore, (\ref{E;dq}) could be a better approximation of $\|\delta(t)\|$.

\begin{pr}[sufficiency] \label{suff}
For all $t$, 
\begin{align*}
\|x_e(t)-x_e(t_k)\|\le \frac{\epsilon_1}{\epsilon_1 +\frac{\beta^2(t)}{4\alpha}}\frac{\beta(t)}{4\alpha}
\end{align*}
 ensures $\|e(t)\|\le \epsilon$, for $\epsilon_1=\frac{c_2c_4}{c_1c_3}\epsilon$. 
\end{pr}
\begin{proof}
From (\ref{E;dq}), $\alpha\|x_e(t)-x_e(t_k)\|^2+\beta(t)\|x_e(t)-x_e(t_k)\| \le \epsilon_1$ implies $\|\delta\|\le \epsilon_1$.

Therefore, 
$$\|x_e(t)-x_e(t_k)\|\le\frac{\sqrt{\beta(t)^2+4\alpha\epsilon_1}-\beta(t)}{2\alpha}$$
implies $\|\delta\|\le \epsilon_1$. 
Using the fact $$\ln(1+x)\ge \frac{x}{1+x}$$ for all $x\ge 0$, one can verify
\begin{align*}
\sqrt{\beta(t)^2+4\alpha\epsilon_1}\ge \beta(t)+\frac{2\alpha\beta(t)\epsilon_1}{(\beta(t)^2+4\alpha\epsilon_1)}.
\end{align*}
Thus, $$\|x_e(t)-x_e(t_k)\|\le\frac{\beta(t)\epsilon_1}{(\beta(t)^2+4\alpha\epsilon_1)}$$
ensures $\|\delta\|\le \epsilon_1=\frac{c_2c_4}{c_1c_3}\epsilon$.

Proposition \ref{pr:e} ensures that $\|\delta\|\le \epsilon_1=\frac{c_2c_4}{c_1c_3}\epsilon$ implies $\|e(t)\|\le \epsilon$.
\end{proof}

As a comparison, in \cite{maity2015cdc}, the sufficient condition equivalent to Proposition \ref{suff} was $\|x_e(t)-x_e(t_k)\|\le \frac{\epsilon_1}{\kappa(t)}$. If $$\epsilon_1< \frac{\beta(\kappa-\beta)}{4\alpha}$$ the sufficiency condition in Proposition \ref{suff} is relaxed than its counterpart in \cite{maity2015cdc}.

 The following lemma ensures that the proposed event-triggering mechanism excludes Zeno behavior.

\begin{lm}
The inter-trigger time $\tau_k=t_k-t_{k-1}$  is bounded from below, i.e. $\inf_k \tau_k \ge \alpha >0$. This ensures that within a finite interval $[t_0,T)$ there will be a finite number of triggerings.
\end{lm}

The lemma can be proved following the approach of \cite[Theorem 3.1]{maity2015cdc}; we omit it due to space limitation.
 
\begin{thm}[\textbf{Main Result}]
If there exists $\epsilon>0$ and controllers $\gamma_i(\cdot)$ such that the closed-loop trajectory $x[t_0]\models \varphi^\epsilon$, then the event triggered trajectory $x_e[t_0] \models \varphi$ where events are generated whenever $g(t)\le 0$.
\end{thm} 
 
The proof follows directly from Proposition (\ref{Pr:robust}) where $y[t_0]=x_e[t_0]$ and we have ensured $\sup_t\rho(x(t),x_e(t))=\sup_t\|e(t)\| \le \epsilon$. 
 
 \subsection{Implication of Delays}
 
 In this section we study the scenario when the sampled state $x(t_k)$ arrives to the controller at time $t_k+\Delta_k$ where $\Delta_k$ is the delay in the channel at time $t_k$. The aim of this section is to find a bound on the the delays so that the proposed event-triggered strategy still ensures that $\|e(t)\|\le \epsilon$ for all time $t$.
 
 In order to study that we start with the sufficiency condition $\|x_e(t)-x_e(t_k)\| \le h(\epsilon)$ which ensures $\|e(t)\|\le \epsilon$. Here $h(\epsilon)$ is $\frac{\beta(t)\epsilon_1}{(\beta(t)^2+4\alpha\epsilon_1)}$ (or $\frac{\epsilon_1}{\kappa(t)}$ by \cite{maity2015cdc}).
 
Let the delays at triggering times  $t_{k-1}$ and $t_k$ be $\Delta_{k-1}$, $\Delta_k$. Thus for all $t\in [t_{k-1}+\Delta_{k-1},t_k+\Delta_k)$ the requirement is
$$\|x_e(t)-x_e(t_{k-1})\| \le h(\epsilon)$$
Now,
\begin{align*}
\|x_e(t)-&x_e(t_{k-1})\| \le \\ &\int_{t_{k-1}}^t\|f_0(s,x)+\sum_{i=1}^mf_i(s,x)\gamma_i^e(x(t_{k-1}))\|ds \\
\le  &\int_{t_{k-1}}^t(l\|x_e(s)-x_e(t_{k-1})\|+\|p(s)\|)ds \\
\end{align*}
where $l=L^0_f+\sum_{i=1}^mL^i_f\gamma^e_i(x(t_{k-1}))$ and $p(s)=f_0(s,x(t_{k-1}))+\sum_{i=1}^mf_i(s,x(t_{k-1}))\gamma_i^e(x(t_{k-1}))$.

Therefore,
\begin{align*}
\|x_e(t)-&x_e(t_{k-1})\| \le \left(\int_{t_{k-1}}^t\|p(s)\|ds\right)e^{l(t-t_{k-1})}
\end{align*}

Let $T_k=\inf_t\{t>t_{k-1}| \left(\int_{t_{k-1}}^t\|p(s)\|ds\right)e^{l(t-t_{k-1})}=h(\epsilon)\}$. 

Therefore we must have $T_k\ge t_k+\Delta_k$ and $t_k\ge t_{k-1}+\Delta_{k-1}$.
Thus, $$\Delta_k+\Delta_{k-1}\le T_k-t_{k-1}$$
From the definition of $T_k$, 
\begin{align}
 \left(\int_{0}^{T_k-t_{k-1}}\|p(s+t_{k-1})\|ds\right)e^{l(T_k-t_{k-1})}=h(\epsilon)=\bar \epsilon
\end{align}
It is trivial to verify the above equation has unique solution for $T_k-t_{k-1}$ whenever $\bar{\epsilon}\ge 0$, and let us denote this solution by $T_k-t_{k-1}=\tilde w(\bar\epsilon)$ for some function $\tilde{w}$ which satisfies the differential equation:
\begin{align} \label{E:w}
\frac{d\tilde w(r)}{dr}&=\frac{1}{rl+e^{l\tilde w}\|p(\tilde{w}+ t_{k-1})\|} \\ \tilde{w}(0)&=0 \nonumber.
\end{align}
When $f_i$ does not depend explicitly on time, then $p(s)=p$ and for this special case, 
\begin{align}
\tilde{w}(\bar \epsilon)=\frac 1lW(l\bar{\epsilon}/p) 
\end{align}
where $W(\cdot)$ is the Lambert $W$ function. 

From (\ref{E:w})  one can verify that for all $r>0$, $\tilde{w}(r)>0$. Moreover using comparison lemma \cite[Lemma 3.4]{khalil1996nonlinear}, one can show for (\ref{E:w}) that 
$$\tilde{w}(r)\ge \frac 1lW(lr/p_m)$$ where $p_m=\sup \|p(\tilde{w}+t_{k-1})\|$.
Using the concavity property of $W(\cdot)$ along with $W(0)=0$,  for all $0 \le r\le r_m$
$$\tilde{w}(r)\ge \frac{W(lr_m/p_m)}{lr_m} r$$. 

Thus, for $\bar{\epsilon}_m\ge\bar{\epsilon}=h(\epsilon)$, 
\begin{align} \label{E:bd}
\sup_k\{\Delta_k+\Delta_{k-1}\}\le \frac{W(l\bar{\epsilon}_m/p_m)}{l\bar{\epsilon}_m} h({\epsilon})
\end{align}
ensures that $\|e(t)\|\le\epsilon$. Thus, (\ref{E:bd}) states the sufficient condition for delays under which the proposed event trigger mechanism will be able to  ensure $\|e(t)\|\le\epsilon$. 

\section{Examples and Simulations} \label{S:5}
\subsection{Example 1}
 Let us consider the following nonlinear system: 
  \begin{equation}
{\begin{bmatrix}
\dot x_1\\\dot x_2
\end{bmatrix}}=-\begin{bmatrix}
\sin(x_1)\\x_2
\end{bmatrix}+\begin{bmatrix}
-x_2 \\x_1
\end{bmatrix}u
\end{equation}
 where $x_1(0)=0, x_2(0)=1$.
We consider the region $\pi_1 \cong \mathcal{X}_1=B_0(0.1)\subset \mathbb{R}^2$. 
The constraint (requirement) is to guide the state of the system within $B_0(0.1)$  and keep the trajectory within that ball for all future times.
The logical constraint is represented by  $\Diamond\Box \pi_1$.  To proceed, we consider the $\epsilon$-robust formula $\pi_1^\epsilon$ with  $\epsilon=0.05$. Therefore, $\mathcal{X}_1^\epsilon=B_0(0.05)$.

We chose the controller $u=-x_2$ that achieves the property that $x[0]\models \pi_1^\epsilon$ (use the Lyapunov function $V=x_1^2+x_2^2$ to verify).  

At this stage we need to design an event triggering mechanism that will ensure that the event triggered system will follow the actual trajectory $x[t_0]$. The event triggered system dynamics is given by:
\begin{align}
{\begin{bmatrix}
\dot x_1\\\dot x_2
\end{bmatrix}}=-\begin{bmatrix}
\sin(x_1)\\x_2
\end{bmatrix}-\begin{bmatrix}
-x_2 \\x_1
\end{bmatrix}x_2(t_k)
\end{align}
 The initial condition is given as $(0,1)$.
One can check that $\bar M_i, L^i_\gamma$ defined in Proposition \ref{suff} have value $1$. In Figure \ref{F:e11}, the error signal is shown  along with the triggering instances. From Figure \ref{F:e11} we note that only  $9$ samples are needed to achieve the task. This requires drastically reduced communication and sensing when compared to the continuous time feedback system. The trajectories of the ideal and even-triggered systems, and the $\epsilon$-tube are shown in Figure \ref{F:e11_2}.

\begin{center}
\begin{figure}
\includegraphics[width=0.5 \textwidth]{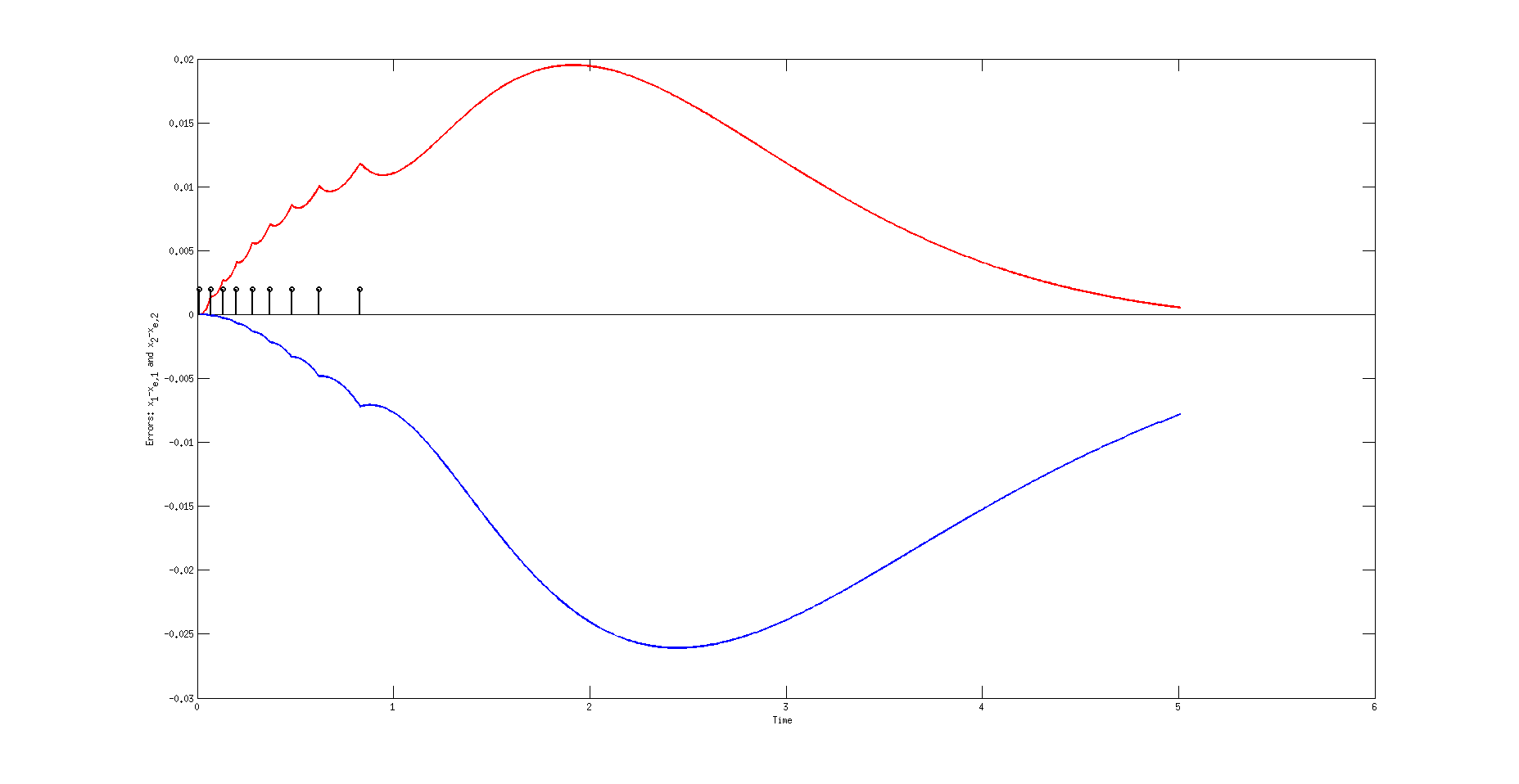} 
\caption{The red curve corresponds to the first component of the error $e=x_e-x$ and the blue one corresponds to the second component. The plot also shows the triggering instances. At each triggering instance, we notice corrective changes in the error components. Eventually the error components go to zero but is not shown here.}\label{F:e11}
\end{figure}
\end{center}

\begin{figure}
\begin{center}
\includegraphics[height = 7 cm]{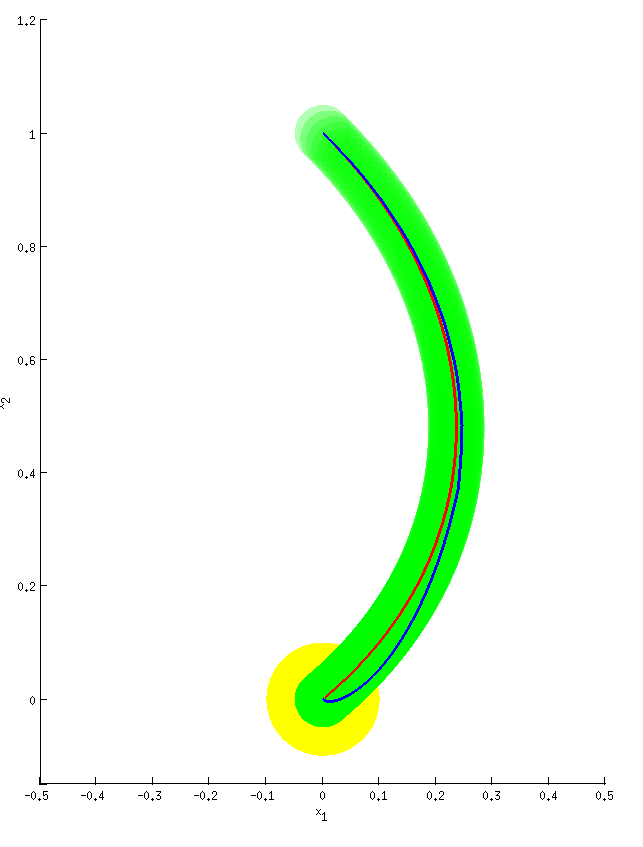}
\caption{The red trajectory corresponds to continuous feedback control $x[0]$ and the blue trajectory corresponds to event triggered control. The green tube has a radius $0.05$ and it shows that the proposed event triggered control ensures the trajectory is within that tube. The yellow circle corresponds to the given rule that the trajectory must be confined in there eventually.}\label{F:e11_2}
\end{center}
\end{figure}

\subsection{Example 2}
In this example we consider a robotic motion planning task with temporal sequencing and obstacle avoidance. The robot dynamics considered here is a unicycle model as:
\begin{equation}
\begin{bmatrix}
\dot x \\ \dot y \\ \dot \theta
\end{bmatrix} = \begin{bmatrix}
\cos\theta \\ \sin\theta \\0
\end{bmatrix}v +
\begin{bmatrix}
0\\0\\1
\end{bmatrix}w
\end{equation}
where $x,y\in \mathbb{R}^2$ is the physical position  and $\theta \in [0,360^o)$ is the heading angle. 
The task is given as follows:
\begin{equation}
\varphi \models \Diamond \pi_2 \wedge (\neg \pi_2 \textbf{U} \pi_1) \wedge \Box \neg \pi_3
\vspace{-5pt}
\end{equation}
where $\pi_1,\pi_2$ and $\pi_3$ corresponds to three circular regions as shown (denoted by $R_1$, $R_2$ and $R_3$) in Figure \ref{F:ex2}. The RTL formula defines the task of avoiding $R_2$ until reaching $R_1$ and eventually reaching $R_2$, and during the whole time the trajectory should avoid $R_3$.
 We adopt a potential function based approach \cite{rimon1992exact} to generate the control laws for navigating the robot. 
\begin{figure}
\begin{center}
\includegraphics[width=0.45\textwidth]{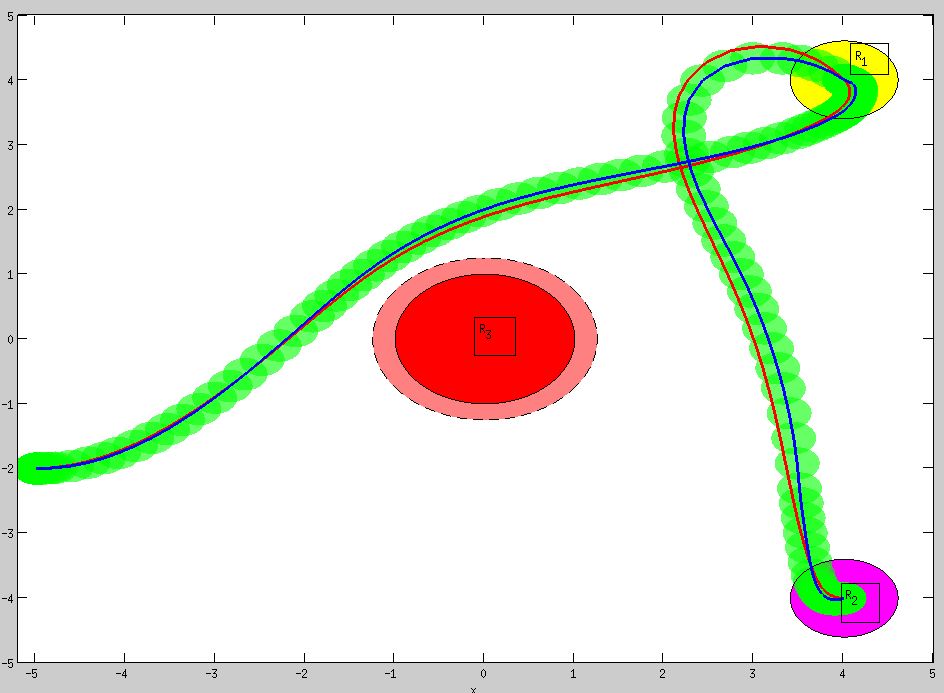}
\caption{The closed loop trajectory is plotted using blue line and the event triggered trajectory with red line. The green tube around the nominal trajectory has radius $0.25$.  The initial position and orientation of the robot is $(-5,-2,0)$} \label{F:ex2}
\end{center}
\end{figure}
As presented in previous sections, we expand and contract the appropriate regions while synthesizing the control.
The region $R_3$ has been expanded by $\epsilon=0.25$ (the dashed boundary around $R_3$ shows the expansion in Figure \ref{F:ex2}) while $R_1$ was contracted and $R_2$ has been both expanded and contracted (since both $\pi_2$ and $\neg \pi_2$ are present), however, we do not explicitly show them in Figure  \ref{F:ex2}. 
In figure \ref{F:ex2_2}, we show the triggering instances for this problem. 
\begin{figure}
\begin{center}
\includegraphics[width=0.5\textwidth]{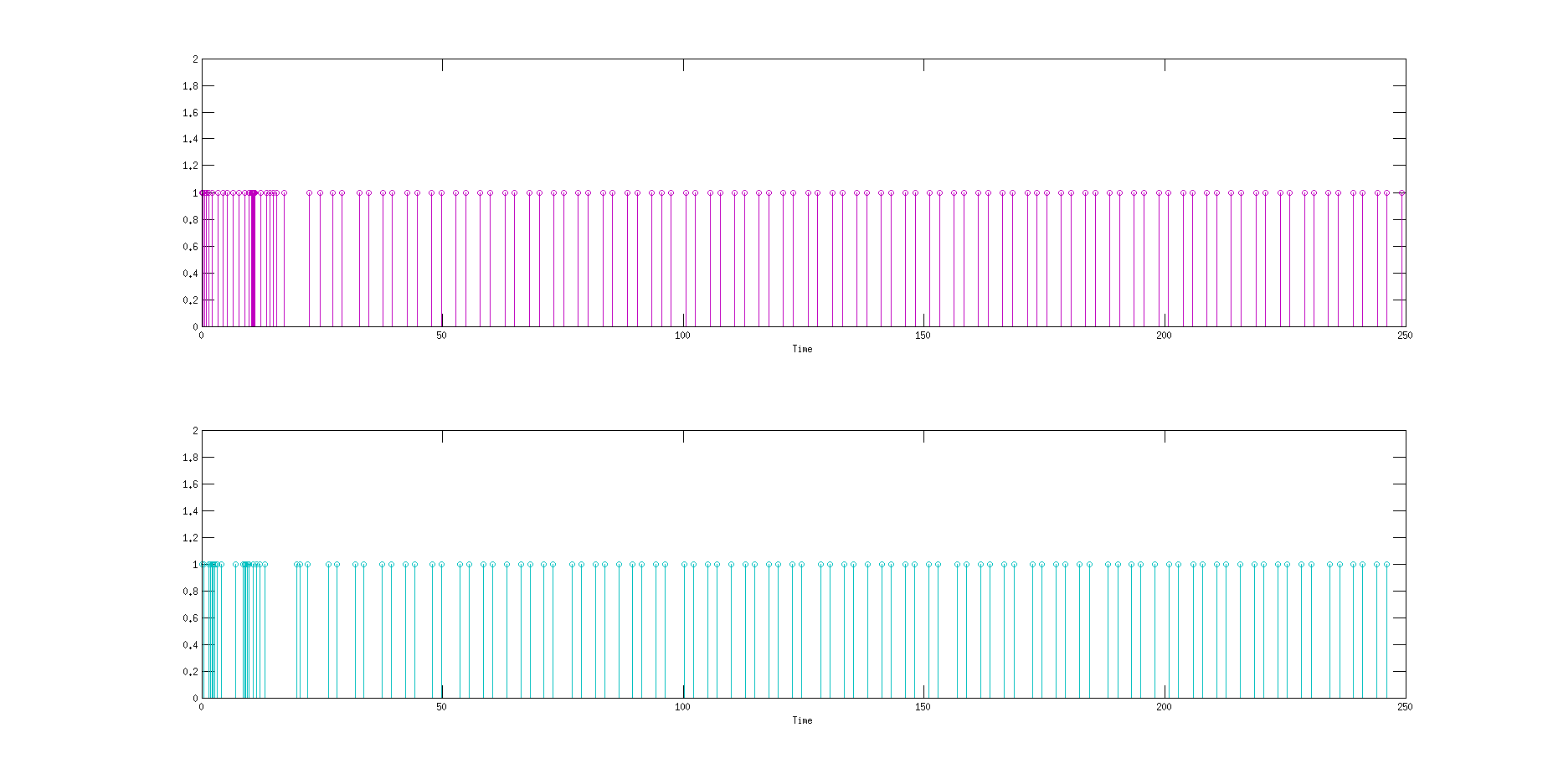}
\vspace{-15 pt}
\caption{The upper graph shows the triggering instances for the trajectory from initial position to $R_1$. The lower graph shows the same for the other segment of the trajectory.} \label{F:ex2_2}
\end{center}
\vspace{-10pt}
\end{figure}

\section{Conclusion}
In this work, we have proposed a framework for integrating the event-triggered controller synthesis and logic based controller synthesis. Our solution is based on composition of two independent controller synthesis framework.
It is also noteworthy that not any pair of logic-based-controller and event-trigger-controller has this unique composability property.
We have derived an explicit event triggering mechanism to bound the trajectory within an $\epsilon$-tube. With the notion of robust logic constraints, the resulting trajectory finally satisfies the logical constraint. 
Simulation results show the significant reduction in communicating the state value for  updating the controller. This reduces the communication and computation costs. 

%

\bibliographystyle{ieeetr}

\bibliography{Bib}

\end{document}